\documentclass{article}
\usepackage[margin=1in]{geometry}
\usepackage{amsmath, amssymb}
\usepackage{graphicx}
\usepackage{booktabs}
\usepackage{enumitem}
\usepackage[ruled,vlined]{algorithm2e}
\usepackage[english]{babel}
\usepackage{newtxmath}

\usepackage{amsthm}
\usepackage{placeins}
\usepackage{subcaption}
\usepackage{mathtools}
\usepackage[numbers,sort&compress]{natbib}
\usepackage{hyperref}
\hypersetup{
    colorlinks=false,
    citecolor=blue,
    pdftitle={Preserving Vector Space Properties in Dimensionality Reduction},
    pdfauthor={Eddi Weinwurm}
}
\usepackage{cleveref}
\newtheorem{theorem}{Theorem}[section]
\newtheorem{lemma}[theorem]{Lemma}
\newtheorem{corollary}[theorem]{Corollary}

\title{Preserving Vector Space Properties in Dimensionality Reduction:\\A Relationship Preserving Loss Framework}
\usepackage{authblk}

\author[1]{Eddi Weinwurm}
\author[2]{Alexander Kovalenko}
\affil[1]{Independent Researcher}
\affil[2]{Department of Applied Mathematics, 
Faculty of Information Technology, 
Czech Technical University in Prague}
\affil[ ]{\texttt{root@grandgreat.com, alexander.kovalenko@fit.cvut.cz}}
\date{August 6, 2025}

\begin{document}
\maketitle

\begin{abstract}
Dimensionality reduction can distort vector space properties such as orthogonality and linear independence, which are critical for tasks including cross-modal retrieval, clustering, and classification. We propose a Relationship Preserving Loss (RPL), a loss function that preserves these properties by minimizing discrepancies between relationship matrices (e.g., Gram or cosine) of high-dimensional data and their low-dimensional embeddings. RPL trains neural networks for non-linear projections and is supported by error bounds derived from matrix perturbation theory. Initial experiments suggest that RPL reduces embedding dimensions while largely retaining performance on downstream tasks, likely due to its preservation of key vector space properties. While we describe here the use of RPL in dimensionality reduction, this loss can also be applied more broadly, for example to cross-domain alignment and transfer learning, knowledge distillation, fairness and invariance, dehubbing, graph and manifold learning, and federated learning where distributed embeddings must remain geometrically consistent.
\end{abstract}

\section{Introduction}
Dimensionality reduction is widely used to compress high-dimensional data for computational efficiency, denoising, visualization, preprocessing, etc., but most existing methods distort essential structural properties such as orthogonality, angles, and linear independence. Methods like PCA \cite{jolliffe2002principal}, t-SNE \cite{van2008visualizing}, and UMAP \cite{mcinnes2018umap} focus on variance or local neighborhoods but often fail to preserve orthogonality, angular relationships, or linear independence, which are essential for tasks like cross-modal retrieval \cite{radford2021learning}. We propose a Relationship Preserving Loss (RPL), a framework that trains projection networks to preserve these properties through customizable relationship matrices, with error bounds derived from matrix perturbation theory. This paper presents the theoretical foundation; full experimental analysis will follow in future work.
\section{Related Work}
Dimensionality reduction methods can be grouped by the structural properties they aim to preserve. 
PCA \cite{jolliffe2002principal} maximizes variance but struggles with non-linear manifolds. 
t-SNE \cite{van2008visualizing} and UMAP \cite{mcinnes2018umap} emphasize local neighborhoods but distort global geometry, neglecting orthogonality and linear independence. 
Kernel PCA \cite{scholkopf1998nonlinear} and Multidimensional Scaling (MDS) \cite{cox2008multidimensional} preserve specific metrics (Gram or distance) through eigendecomposition but lack flexibility. 
Graph-based approaches such as Laplacian Eigenmaps \cite{belkin2003laplacian} preserve connectivity but not angular or independence properties. 
Subspace-preserving methods \cite{vidal2014dimensionality} maintain independent subspaces but are less customizable and often computationally heavier. 
Recent correlation-preserving approaches \cite{gildenblat2025pcc} align pairwise distance correlations (global structure) but do not target orthogonality or rank. 

RPL extends these directions by supporting arbitrary relationship functions and discrepancy measures, learning non-linear neural mappings, and scaling via masking and mini-batch sampling. 
Unlike prior methods, RPL comes with perturbation-theoretic error bounds, ensuring that orthogonality, rank, and subspace structure are preserved up to quantifiable distortion.

\paragraph{Relation to classical MDS and Kernel PCA.}
Classical MDS \cite{borg2005modern} and kernel PCA \cite{scholkopf1998nonlinear} preserve Gram or distance matrices via eigendecomposition. 
RPL generalizes these ideas by: (i) replacing eigendecomposition with a differentiable loss and non-linear neural map, 
(ii) supporting arbitrary relationship functions $\phi$ and discrepancy measures $\mathcal{D}$, and 
(iii) incorporating sparse masking and mini-batch sampling for scalability, with provable error bounds for large datasets.

\begin{table}[ht]
\centering
\caption{Comparison of dimensionality reduction methods. Rows are methods; columns are preserved properties.}
\label{tab:comparison}
\begin{tabular}{lccc}
\toprule
Method & Orthogonality & Linear Independence & Customizable $\phi$ \\
\midrule
PCA \cite{jolliffe2002principal} & Partial & Yes & Yes (via kernel trick) \\
UMAP \cite{mcinnes2018umap} & No & Partial & No \\
Subspace Preservation \cite{vidal2014dimensionality} & Yes & Yes & No \\
PCC \cite{gildenblat2025pcc} & Partial & No & Limited \\
RPL & Yes & Yes & Yes \\
\bottomrule
\end{tabular}
\end{table}

\section{Relationship Preserving Loss}
\subsection{Formulation}
RPL minimizes the discrepancy between relationship matrices of high-dimensional data $\boldsymbol{X} \in \mathbb{R}^{n \times d}$ and its low-dimensional embedding $\boldsymbol{Y} = f(\boldsymbol{X}) \in \mathbb{R}^{n \times k}$, where $f$ is a neural network. Relationship matrices $R(\boldsymbol{X}), \hat{R}(\boldsymbol{Y}) \in \mathbb{R}^{n \times n}$ are defined as
\[
R(\boldsymbol{X})_{ij} = \phi(\boldsymbol{X}_i, \boldsymbol{X}_j), \quad
\hat{R}(\boldsymbol{Y})_{ij} = \phi(\boldsymbol{Y}_i, \boldsymbol{Y}_j),
\]
where $\phi: \mathbb{R}^d \times \mathbb{R}^d \to \mathbb{R}$ is a user-defined function (e.g., dot product for Gram matrices). The loss is
\begin{equation}
\mathcal{L}_{\text{RPL}} = \mathcal{D} \!\left( R(\boldsymbol{X}), \hat{R}(\boldsymbol{Y}) \right),
\end{equation}
where $\mathcal{D}$ measures matrix differences.

\subsection{Relationship Functions}
Options for $\phi$ include:
\begin{itemize}
    \item \textbf{Dot Product:} $\phi(\boldsymbol{X}_i, \boldsymbol{X}_j) = \boldsymbol{X}_i \cdot \boldsymbol{X}_j$, preserving orthogonality and linear relationships.
    \item \textbf{Cosine Similarity:} $\phi(\boldsymbol{X}_i, \boldsymbol{X}_j) = \tfrac{\boldsymbol{X}_i \cdot \boldsymbol{X}_j}{\|\boldsymbol{X}_i\| \|\boldsymbol{X}_j\|}$, preserving angles.
    \item \textbf{Covariance:} $\phi(\boldsymbol{X}_i, \boldsymbol{X}_j) = (\boldsymbol{X}_i - \bar{\boldsymbol{X}}) \cdot (\boldsymbol{X}_j - \bar{\boldsymbol{X}})$, for statistical structure.
    \item \textbf{RBF Kernel:} $\phi(\boldsymbol{X}_i, \boldsymbol{X}_j) = \exp(-\gamma \|\boldsymbol{X}_i - \boldsymbol{X}_j\|^2)$, for non-linear relationships.
\end{itemize}

\subsection{Discrepancy Functions and Masking}
Discrepancy functions $\mathcal{D}$ include:
\begin{itemize}
    \item \textbf{Mean Squared Error:} $\sum_{i,j} \left( R(\boldsymbol{X})_{ij} - \hat{R}(\boldsymbol{Y})_{ij} \right)^2$.
    \item \textbf{Absolute Error:} $\sum_{i,j} \left| R(\boldsymbol{X})_{ij} - \hat{R}(\boldsymbol{Y})_{ij} \right|$.
    \item \textbf{KL Divergence:} with normalized matrices.
\end{itemize}
Masking strategies can emphasize significant relationships:
\begin{itemize}
    \item \textbf{Top-$k$:} select the $k$ largest entries of $R(\boldsymbol{X})$.
    \item \textbf{Sigmoid-Weighted:} $w_{ij} = \sigma(\alpha R(\boldsymbol{X})_{ij})$.
\end{itemize}

\begin{algorithm}
\caption{RPL Training with Mini-Batch Sampling}
\label{alg:rpl}
\KwIn{Data $\boldsymbol{X} \in \mathbb{R}^{n \times d}$, target dimension $k$, functions $\phi$, $\mathcal{D}$, batch size $b$, neural network $f_\theta$}
\KwOut{Mapping $f_\theta: \mathbb{R}^d \to \mathbb{R}^k$}
Initialize $\theta$\;
\For{epoch = 1 to max epochs}{
    \For{mini-batch $\boldsymbol{X}_B \in \mathbb{R}^{b \times d}$ from $\boldsymbol{X}$}{
        Compute $\boldsymbol{Y}_B = f_\theta(\boldsymbol{X}_B)$\;
        Compute $R(\boldsymbol{X}_B)_{ij} = \phi(\boldsymbol{X}_{B,i}, \boldsymbol{X}_{B,j})$\;
        Compute $\hat{R}(\boldsymbol{Y}_B)_{ij} = \phi(\boldsymbol{Y}_{B,i}, \boldsymbol{Y}_{B,j})$\;
        Compute $\mathcal{L}_{\text{RPL}} = \mathcal{D}\!\left(R(\boldsymbol{X}_B), \hat{R}(\boldsymbol{Y}_B)\right)$\;
        Update $\theta$ via gradient descent\;
    }
}
\Return $f_\theta$
\end{algorithm}

\section{Vector Space Guarantees}
\label{sec:vsp}

\paragraph{Setup.}
Let
\[
\Delta \coloneqq R(\mathbf{X}) - \hat{R}(\mathbf{Y}), 
\qquad 
\varepsilon \coloneqq \|\Delta\|_F^2.
\]
During training we observe only a subset
$\mathcal{S} \subset [n] \times [n]$
of size $m$ (uniform without replacement per mini-batch) and record
\[
      \hat{\varepsilon}
      \coloneqq
      \sum_{(i,j) \in \mathcal{S}} \Delta_{ij}^{\,2}.
\]
All results below hold for \emph{any} fixed mini-batch once the stated
probabilities are conditioned on the randomness of $\mathcal{S}$.
When $\phi$ is the dot product, we write the Gram matrices
$G_{\boldsymbol{X}} \coloneqq R(\boldsymbol{X})$ and
$G_{\boldsymbol{Y}} \coloneqq \hat{R}(\boldsymbol{Y})$.

\subsection{Bounding global error}
\label{subsec:serfling}
\begin{lemma}[Serfling transfer]\label{lem:serfling}
Assume $|\Delta_{ij}| \le M$ and fix $\delta \in (0,1)$.
With probability at least $1-\delta$ (over the draw of $\mathcal{S}$)
\[
      \varepsilon
      \le
      \frac{n^2}{m} \, \hat{\varepsilon}
      +
      M^2 \, n^2
      \sqrt{\frac{2 \log(2/\delta)}{m}}.
\]
\end{lemma}
\begin{proof}
Apply Serfling’s inequality \citep{serfling1974probability} to the
variables $Z_{(i,j)} = \Delta_{ij}^{\,2}$,
noting that
$\mathbb{E}[Z_{(i,j)}] = \varepsilon / n^2$ and
$0 \le Z_{(i,j)} \le M^2$. Multiply the deviation bound
for the sample mean by $n^2$ and rearrange.
\end{proof}

\noindent\textit{Practical corollary.}
Driving $\hat{\varepsilon} \to 0$ forces
$\varepsilon \to 0$ once the cumulative number of observed pairs
satisfies $m = \tilde{\Theta}(n^2)$, the regime reached after a
handful of epochs on standard batch schedules.

\subsection{Orthogonality guarantee}
\begin{theorem}[Operator--Lipschitz]\label{thm:oplip}
Let $g: \mathbb{R}^{n \times n} \to \mathbb{R}$ be operator-Lipschitz with
constant $L_g$ in the spectral norm. Then
\[
      |g(R(\boldsymbol{X})) - g(\hat{R}(\boldsymbol{Y}))|
      \le
      L_g \, \sqrt{\varepsilon}.
\]
\end{theorem}
Since $\|\Delta\|_2 \le \|\Delta\|_F = \sqrt{\varepsilon}$, the proof
is immediate.

\paragraph{Orthogonality.}
Taking $g_{ij}(M) = M_{ij}$ ($L_{g_{ij}} = 1$) yields:
\begin{corollary}[Entry-wise orthogonality]\label{cor:orth}
If $\boldsymbol{X}_i \cdot \boldsymbol{X}_j = 0$ then
\[
      |\boldsymbol{Y}_i \cdot \boldsymbol{Y}_j|
      \le
      \sqrt{\varepsilon}.
\]
\end{corollary}

\subsection{Rank preservation}
Denote by
$\lambda_1(\cdot) \ge \dots \ge \lambda_n(\cdot)$
the ordered eigenvalues of a symmetric matrix.
\begin{theorem}[Rank-$r$ preservation]\label{thm:rank}
Let $r = \operatorname{rank}(\boldsymbol{X})$ and
$\sigma_r(\boldsymbol{X})$ be its smallest non-zero singular value.
Assume $k \ge r$ and
\[
      \varepsilon
      <
      \sigma_r^{\,4}(\boldsymbol{X}).
\]
Then
\[
      \lambda_{n-r+1}\bigl(G_{\boldsymbol{Y}}\bigr)
      \ge
      \sigma_r^{\,2}(\boldsymbol{X}) - \sqrt{\varepsilon}
      > 0,
      \quad
      \operatorname{rank}\bigl(\boldsymbol{Y}\bigr) = r .
\]
\end{theorem}
\begin{proof}
Weyl’s inequality gives
\[
      \bigl|\lambda_{n-r+1}(G_{\boldsymbol{X}})
      - \lambda_{n-r+1}(G_{\boldsymbol{Y}})\bigr|
      \le
      \|\Delta\|_2
      \le
      \sqrt{\varepsilon}.
\]
But
$\lambda_{n-r+1}(G_{\boldsymbol{X}})=\sigma_r^{\,2}(\boldsymbol{X})$.
Positivity is therefore guaranteed when
$\varepsilon < \sigma_r^{\,4}(\boldsymbol{X})$.
\end{proof}

\subsection{Subspace preservation}
Let
$\mathcal{U}=\operatorname{range}\bigl(G_{\boldsymbol{X}}\bigr)$ and
$\mathcal{V}=\operatorname{range}\bigl(G_{\boldsymbol{Y}}\bigr)$,
both subspaces of $\mathbb{R}^n$ of dimension $r$.
Let $\Theta$ be their largest principal angle.
\begin{theorem}[Davis--Kahan angle]\label{thm:dk}
Under the hypotheses of Theorem~\ref{thm:rank},
\[
      \sin \Theta
      \le
      \frac{\sqrt{\varepsilon}}{\sigma_r^{\,2}(\boldsymbol{X})} .
\]
\end{theorem}
\begin{proof}
Apply the Davis--Kahan $\sin \Theta$ theorem
(e.g.\ \citealt{hsu2016dk}) to the eigenspaces of
$G_{\boldsymbol{X}}$ and $G_{\boldsymbol{Y}}$,
using $\|\Delta\|_2 \le \sqrt{\varepsilon}$.
\end{proof}

\subsection*{Remarks}
\begin{enumerate}[label=\arabic*., leftmargin=2em]
\item Bounds depend only on
    $\sqrt{\varepsilon}$ and the intrinsic scale
    $\sigma_r^{\,2}(\boldsymbol{X})$; no spurious constants appear.
\item Rank guarantees are stated in terms of the
    $(n-r+1)$-st eigenvalue, matching the fact that
    $G_{\boldsymbol{X}}$ has exactly $r$ positive eigenvalues when $r<n$.
\item Angles are measured between the ranges of the two Gram matrices, which both live in the common ambient space $\mathbb{R}^n$, so principal angles are well-defined even when $k \ll d$.
\end{enumerate}

\subsection{Kernel extensions }
\label{sec:kernel}
The guarantees above assume $\phi$ is the dot product, yielding the Gram matrix and direct connections to ambient vector-space properties. However, they extend to any symmetric relationship function $\phi$ for which $R(\boldsymbol{X})$ is positive semidefinite (PSD), including RBF and polynomial kernels. In such cases, Theorem~\ref{thm:oplip} and Corollary~\ref{cor:orth} apply directly, preserving entry-wise relationships (e.g., if $\phi(\boldsymbol{X}_i,\boldsymbol{X}_j)=0$ then $|\phi(\boldsymbol{Y}_i,\boldsymbol{Y}_j)|\le \sqrt{\varepsilon}$).

For rank and subspace guarantees, Theorems~\ref{thm:rank} and~\ref{thm:dk} hold with $G_{\boldsymbol{X}}, G_{\boldsymbol{Y}}$ replaced by $R(\boldsymbol{X}), \hat{R}(\boldsymbol{Y})$, $r=\operatorname{rank}(R(\boldsymbol{X}))$, and $\sigma_r^{\,2}(\boldsymbol{X})$ replaced by the smallest non-zero eigenvalue $\lambda_r(R(\boldsymbol{X}))$. Here, the preserved rank refers to the effective dimension of the feature space induced by $\phi$ (e.g., the RKHS rank for kernels), and the subspace distortion is between the eigenspaces of the relationship matrices in $\mathbb{R}^n$. These extensions rely on the same perturbation tools (Weyl, Davis--Kahan).

For kernels like cosine similarity or RBF, explicit Lipschitz constants may require bounded norms or distances. The following table lists conservative upper bounds $L$ relating kernel differences to embedding differences in bounded domains:

\begin{center}
\small
\begin{tabular}{lcc}
\toprule
Kernel & Example $L$ (Upper) & Notes\\
\midrule
Dot product $\boldsymbol{u}^\top \boldsymbol{v}$ & $R$ & $\|\boldsymbol{u}\|, \|\boldsymbol{v}\| \le R$ \\
Cosine $\frac{\boldsymbol{u}^\top \boldsymbol{v}}{\|\boldsymbol{u}\| \|\boldsymbol{v}\|}$ & $2/R_{\min}$ & $\|\boldsymbol{u}\|, \|\boldsymbol{v}\| \ge R_{\min} > 0$\\
RBF $\exp(-\gamma \|\boldsymbol{u}-\boldsymbol{v}\|^2)$ & $\sqrt{2\gamma / e}$ & Heuristic; depends on data range\\
\bottomrule
\end{tabular}
\end{center}

These constants are conservative upper bounds. For cosine, $2/R_{\min}$ arises from bounding derivatives under norm constraints. For RBF, $\sqrt{2\gamma/e}$ reflects the maximal slope of $\exp(-\gamma \|\mathbf{u}-\mathbf{v}\|^2)$. They ensure Lipschitz continuity on compact domains, though they may be loose in practice.

Note: these $L$ upper-bound kernel sensitivity to embedding perturbations. Inverse bounds (deriving embedding closeness from kernel closeness) require bi-Lipschitz assumptions on compact domains.

\section{Experimental Validation}
\label{sec:exp-results}

\paragraph{Setup.}
We evaluate RPL along two axes:
(i) quantitative retrieval performance on MS COCO 2017, and
(ii) qualitative preservation of manifold structure in synthetic stress–tests designed to expose angular distortions and foldovers.

\subsection{Cross–modal retrieval on MS COCO}
We compress ViT-H/14@336 embeddings (1024-D) to lower dimensions and evaluate cross–modal retrieval. 
\Cref{tab:results} reports Recall@$K$, median rank, and MRR@10 for the 1024-D baseline, a 768-D RPL projection, and ViT-L/14@336 (768-D) as reference. 
Despite a 25\% reduction in dimensionality, the RPL projection slightly improves Recall@1 (0.466 vs.\ 0.464) and MRR@10 (0.574 vs.\ 0.573), while leaving higher-$K$ retrieval unchanged. 
Median rank remains optimal (1.0). 
This indicates that RPL compression can preserve or even enhance retrieval quality.

\begin{table}[h]
\centering
\caption{Cross–modal retrieval on MS COCO.}
\label{tab:results}
\begin{tabular}{lccccccc}
\toprule
Embedding & Dim. & R@1 & R@5 & R@10 & R@100 & Med. Rank & MRR@10 \\
\midrule
ViT-H/14@336 & 1024 & 0.464 & 0.722 & 0.812 & 0.988 & 1.0 & 0.573 \\
RPL ViT-H/14@336 & 768 & 0.466 & 0.722 & 0.812 & 0.988 & 1.0 & 0.574 \\
ViT-L/14@336 & 768 & 0.386 & 0.633 & 0.737 & 0.978 & 2.0 & 0.492 \\
\bottomrule
\end{tabular}
\end{table}

We then compress more aggressively to 256-D (a $4\times$ reduction). 
\Cref{tab:vh_results} shows that naive RPL still achieves Recall@1 $\approx$ 0.453, close to the baseline. 
Importantly, Top-$k$ masking recovers performance (R@1 = 0.456), maintaining parity with the uncompressed model at larger $K$. 
Alternative weightings yield comparable but not superior results. 
These results confirm that RPL can sustain retrieval quality under substantial compression, especially when guided by relationship masking.

\begin{table}[h]
\centering
\caption{Masking strategies for ViT-H/14@336 compressed to 256-D.}
\label{tab:vh_results}
\begin{tabular}{lccccccc}
\toprule
Masking & Dim. & R@1 & R@5 & R@10 & R@100 & Med. Rank & MRR@10 \\
\midrule
Original & 1024 & 0.464 & 0.722 & 0.812 & 0.988 & 1.0 & 0.573 \\
None & 256 & 0.453 & 0.709 & 0.803 & 0.987 & 1.0 & 0.562 \\
Top-$k$ & 256 & 0.456 & 0.711 & 0.806 & 0.987 & 1.0 & 0.564 \\
Weighted & 256 & 0.453 & 0.707 & 0.806 & 0.988 & 1.0 & 0.563 \\
Linear & 256 & 0.433 & 0.694 & 0.793 & 0.987 & 1.0 & 0.544 \\
Gaussian & 256 & 0.453 & 0.710 & 0.805 & 0.988 & 1.0 & 0.562 \\
\bottomrule
\end{tabular}
\end{table}

\subsection{Qualitative manifold preservation}
To probe preservation of vector–space structure beyond retrieval numbers, we embed synthetic 2D manifolds into $\mathbb{R}^{24}$ and train a three–layer MLP with RPL to project into $\mathbb{R}^{3}$. Colors encode a latent parameter common to each row; distortions manifest as mixed gradients or foldovers.

\begin{figure*}[t]
  \centering

  \begin{subfigure}[t]{0.31\textwidth}
    \centering
    \includegraphics[width=\linewidth]{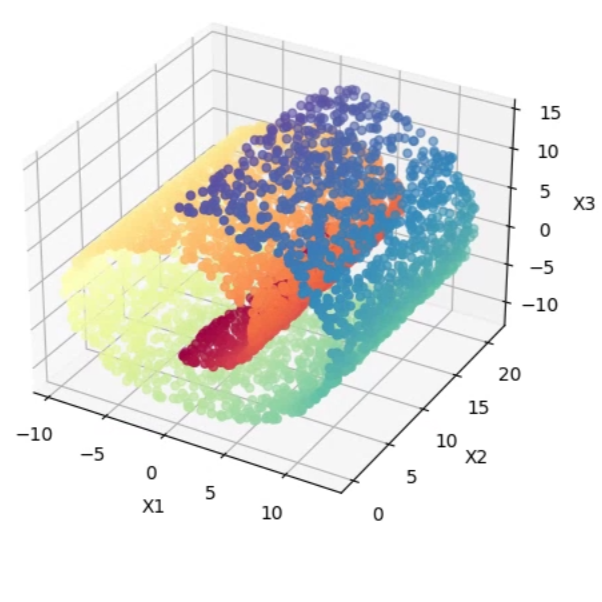}
    \caption{Dataset A: Original}
    \label{fig:cin-orig}
  \end{subfigure}\hfill
  \begin{subfigure}[t]{0.31\textwidth}
    \centering
    \includegraphics[width=\linewidth]{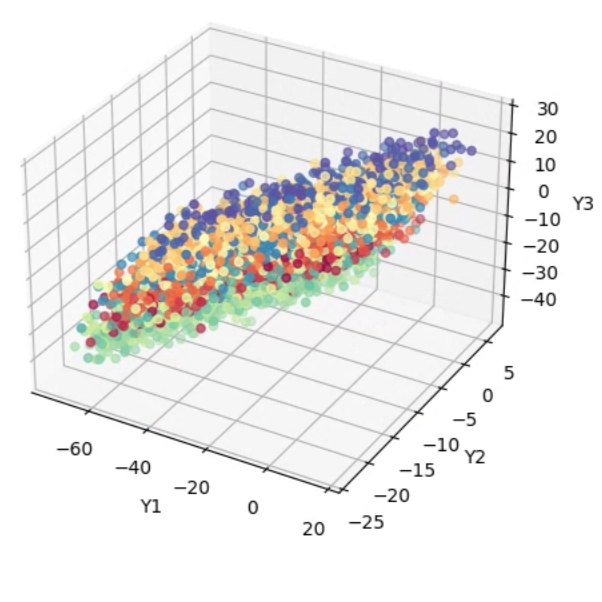}
    \caption{Dataset A: Random network}
    \label{fig:cin-rnd}
  \end{subfigure}\hfill
  \begin{subfigure}[t]{0.31\textwidth}
    \centering
    \includegraphics[width=\linewidth]{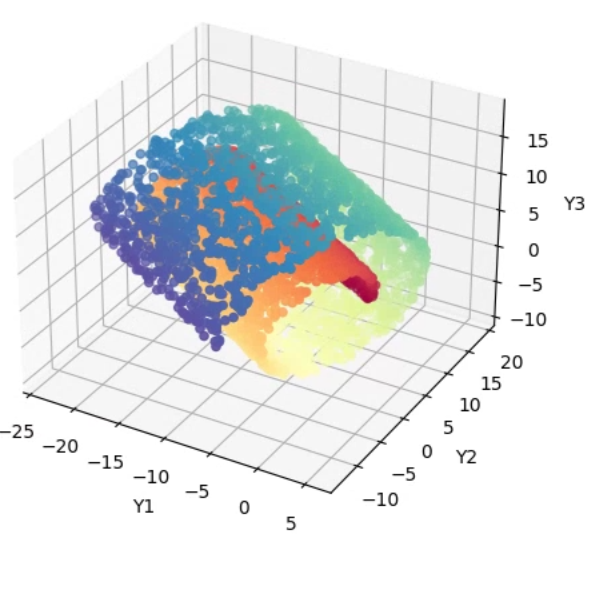}
    \caption{Dataset A: RPL-trained}
    \label{fig:cin-learn}
  \end{subfigure}

  \vspace{0.8em}

  \begin{subfigure}[t]{0.31\textwidth}
    \centering
    \includegraphics[width=\linewidth]{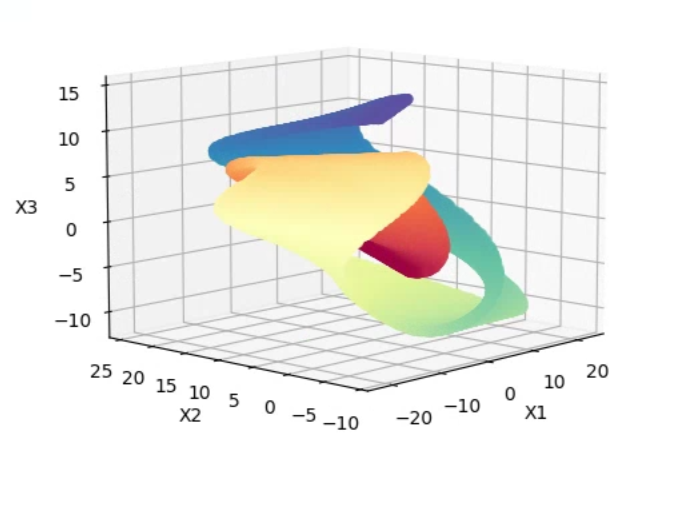}
    \caption{Dataset B: Original}
    \label{fig:comp-orig}
  \end{subfigure}\hfill
  \begin{subfigure}[t]{0.31\textwidth}
    \centering
    \includegraphics[width=\linewidth]{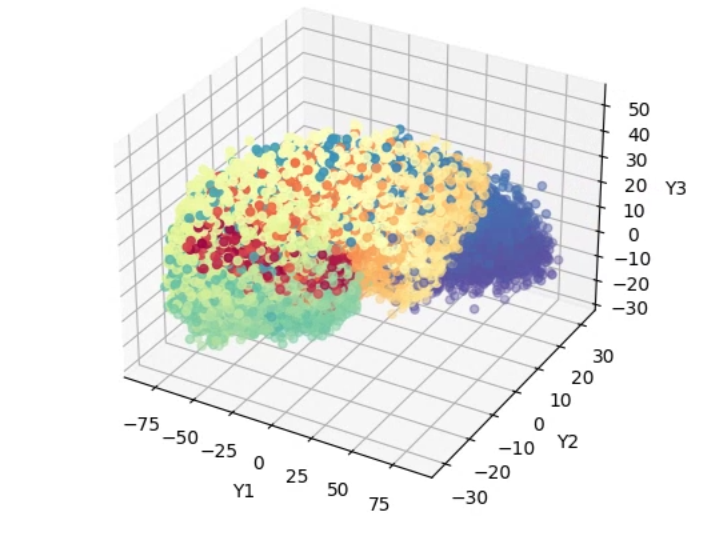}
    \caption{Dataset B: Random network}
    \label{fig:comp-rnd}
  \end{subfigure}\hfill
  \begin{subfigure}[t]{0.31\textwidth}
    \centering
    \includegraphics[width=\linewidth]{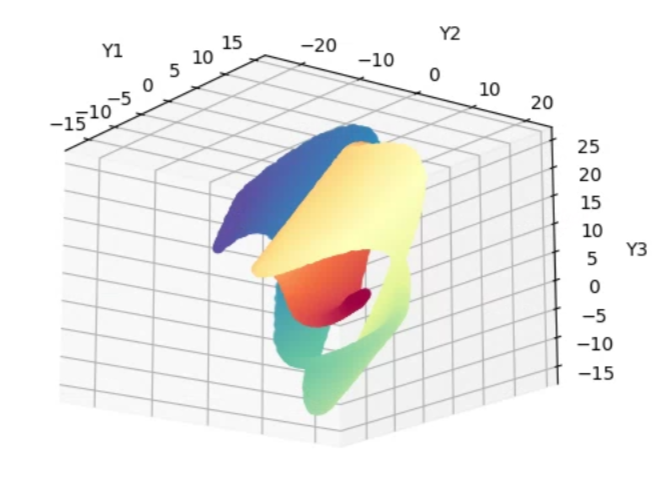}
    \caption{Dataset B: RPL-trained}
    \label{fig:comp-learn}
  \end{subfigure}

  \caption{\textbf{Qualitative manifold projections.}
  Rows: datasets A (top) and B (bottom).
  Columns: original manifold; projection from a randomly initialized network; and RPL-trained projection.
  Colors encode a latent parameter to reveal angular coherence and foldovers.
  RPL-trained projections preserve global ordering and suppress distortions compared to random networks; in Dataset~B the RPL result is a mirrored realization of the same manifold.}
  \label{fig:rpl-subfigs}
\end{figure*}

\paragraph{Interpretation.}
For Dataset~A (the ``cinnamon roll''), RPL recovers the rolled geometry and latent ordering, avoiding the collapse seen in random networks.
For Dataset~B (the twisted surface), RPL faithfully reconstructs the global topology but in a mirrored orientation relative to the original.
This invariance is expected: $\mathcal{L}_{\text{RPL}}$ is insensitive to the full orthogonal group $O(k)$, which includes both rotations (determinant $+1$) and reflections (determinant $-1$).
Thus RPL preserves manifold topology and relative relationships but not absolute orientation or handedness, consistent with the perturbation-theoretic guarantees in \cref{sec:vsp}.

\section{Conclusion}
The Relationship Preserving Loss (RPL) framework advances dimensionality reduction by providing a flexible, neural network-based approach to preserving key vector space properties, including orthogonality, linear independence, and subspace structure. By minimizing discrepancies in customizable relationship matrices (e.g., Gram, cosine, or kernel-based), RPL ensures that non-linear projections maintain these properties with quantifiable fidelity. Our theoretical analysis, grounded in matrix perturbation theory, establishes rigorous error bounds: Serfling's inequality links mini-batch observations to global Frobenius error; operator-Lipschitz continuity guarantees entry-wise preservation (e.g., orthogonality within $\sqrt{\varepsilon}$); Weyl's inequality secures rank preservation when $\varepsilon < \sigma_r^{\,4}(\boldsymbol{X})$; and the Davis--Kahan theorem bounds subspace distortion by $\sin \Theta \leq \sqrt{\varepsilon}/\sigma_r^{\,2}(\boldsymbol{X})$. These results extend kernel-agnostically to PSD relationship functions, broadening applicability while leveraging the same spectral tools.
By integrating sparse masking and mini-batch scalability, RPL not only generalizes classical methods like MDS and kernel PCA but also enables practical deployment in high-dimensional settings. While preliminary empirical checks suggest compression without performance degradation, the core contribution lies in the mathematical framework's provable guarantees, which provide a foundation for reliable embeddings in tasks demanding structural integrity. Future work includes extensive experiments, comparisons with other methods, and scalability tests.


\end{document}